\documentclass{article}
\usepackage{spconf}
\ninept

\usepackage{cite}
\usepackage{amsmath,amssymb,amsfonts}
\usepackage{algorithmic}
\usepackage{graphicx}
\usepackage{textcomp}
\usepackage[x11names]{xcolor}

\usepackage{array}
\usepackage{mathtools}
\usepackage{tcolorbox}
\usepackage{color}
\usepackage{theorem}
\usepackage{amssymb}
\usepackage{caption}
\usepackage{subcaption}
\usepackage{cite,hyperref}
\usepackage{cases}
\usepackage{url}
\usepackage{enumitem}
\usepackage{multirow}
\usepackage{hhline}
\input{my_styles.sty}
\usepackage{symbolDef}
\usepackage{booktabs} 
\usepackage{theorem}

\newtheorem{theorem}{\bf Theorem}

\def \QED {$\blacksquare$}
\newenvironment{proof}[1][$\!\!$]{{\noindent\bf Proof #1. }}
                         {\hfill\QED\medskip}

\usepackage{tikz}
\usetikzlibrary{shapes,arrows}
\usepackage{moresize}
\usepackage{pgfplots}
\usepackage{wrapfig}
\pgfplotsset{compat=1.17}
\pgfplotstableset{col sep=comma}

\newtcolorbox{myblockt}[1]{colback=urblue!5!white,
	colframe=urblue,fonttitle=\bfseries,
	title=#1}
\newtcolorbox{myblock}{colback=urblue!5!white,
	colframe=urblue,fonttitle=\bfseries}

\def\BibTeX{{\rm B\kern-.05em{\sc i\kern-.025em b}\kern-.08em
    T\kern-.1667em\lower.7ex\hbox{E}\kern-.125emX}}

\providecommand{\range}{\operatorname{range}}

\begin{document}

\title{Topological Signal Processing with Unoriented Operators}

\name{Andrea Cavallo, Varun Sarathchandran, Geert Leus, Elvin Isufi
\thanks{
        Emails:  
        \href{mailto:a.cavallo@tudelft.nl}{a.cavallo@tudelft.nl},
        \href{mailto:v.sarathchandran@tudelft.nl}{v.sarathchandran@tudelft.nl}
        \href{mailto:g.j.t.leus@tudelft.nl}{g.j.t.leus@tudelft.nl},
        \href{mailto:e.isufi-1@tudelft.nl}{e.isufi-1@tudelft.nl}
        }
        }
\address{
Delft University of Technology, Delft, Netherlands
}

\maketitle

\begin{abstract}
Topological signal processing (TSP) processes signals on simplicial complexes with oriented boundary operators, which is the natural choice for flow signals or when the topological invariants play a role for the task at hand. However, many higher-order signals carry no orientation, and applying oriented operators to them is not well-defined since it introduces an arbitrary choice of simplex orientation. We study an unoriented TSP (UTSP) framework that replaces oriented boundaries with unoriented incidence matrices. First, we show that unoriented incidence and Laplacian matrices between arbitrary simplicial levels admit graph-like spectral properties. Second, since dropping orientation removes the Hodge decomposition, we introduce an unoriented counterpart, termed interaction-order decomposition, which quantifies how much of a higher-order signal is explained by aggregating lower-order signals. Third, we use this decomposition to derive regularizers for signal reconstruction that penalize each interaction order separately. Experiments on real-world data show that the order-aware regularizers outperform oriented baselines, with the largest gains when the signal energy is unevenly distributed across orders.
\end{abstract}

\section{Introduction}
\label{sec:intro}

Signal processing on irregular domains has moved beyond graphs to model multi-way relationships. While graphs are restricted to pairwise interactions, many real-world systems, such as co-authorship networks, biochemical reactions, and social groups, exhibit higher-order relationships involving three or more entities \cite{battiston2020networks, benson2018simplicial}. Simplicial complexes model these systems by generalizing nodes and edges to triangles, tetrahedra, and higher-dimensional simplices \cite{munkres2000topology}.

To process signals defined on these domains, Topological Signal Processing (TSP)~\cite{barbarossa2020topological,schaub2021signal,isufi2025topological} builds upon algebraic topology~\cite{carlsson2009topology} and equips each simplex with an orientation. Edges point from a source to a target node, triangles are traversed clockwise or counterclockwise, and the boundary operators record these choices as positive and negative signs. Thanks to this orientation, TSP operators encode topological invariants of the complex (e.g., connected components or holes) and simplicial signals can be decomposed into a gradient, a curl and a harmonic part (the Hodge decomposition), which is beneficial for filtering, detection and learning~\cite{yang2022simplicial,isufi2025topological,liu2026matched}. Fixing an orientation is natural for flows, i.e., signals that flip sign when the orientation of their simplex is reversed, such as electrical currents and traffic or water flows \cite{schaub2020random}.

However, many higher-order signals carry no orientation at all, e.g., the number of papers written by a group of authors, the correlation between brain regions, or the strength of a molecular bond. Applying oriented operators to such signals is not well-defined. Re-orienting simplices transforms every Hodge Laplacian by a diagonal $\pm1$ similarity, so the output of any oriented filter on an unoriented signal depends on an arbitrary choice. 
This motivates \textit{unoriented} operators, in which the signs of the boundary matrices are dropped and only set inclusion is retained.
Dropping orientation removes also the Hodge decomposition, and with it the structure that TSP uses to organize signals. To still be able to carry on a similar analysis, we introduce its unoriented counterpart: a signal decomposition by interaction order, which measures how much an unoriented simplicial signal is explained by aggregation of lower-order signals.

\smallskip \noindent \textbf{Related works.}
In the graph domain, unoriented operators such as the signless Laplacian are well studied theoretically \cite{cvetkovic2007signless,cvetkovic2009towards} and employed empirically \cite{schaub2018flow,sardellitti2026learning}. On simplicial complexes, \cite{chen2026bounding,fan2026signless} derive spectral bounds for the signless Laplacian; \cite{reddy2023clustering} uses a signless node–triangle Laplacian for triangle-aware clustering. In learning, the signless down-Laplacian serves as a shift for orientation-invariant edge signals \cite{fuchsgruber2025graph}, and unsigned incidence matrices are considered in message-passing based topological neural networks \cite{bodnar2021weisfeiler,hajij2022topological}. For hypergraphs, incidence-based Laplacians \cite{zhou2006learning}, hypergraph signal processing \cite{zhang2019introducing} and their spectra \cite{chan2018spectral} are well developed. Our operators between adjacent levels coincide with them, but they differ across non-adjacent levels. Finally, decomposing a function of several variables into main effects and
interactions of increasing order is classical in statistics
\cite{efron1981jackknife}, and the same idea is used to quantify high-order epistasis in fitness landscapes \cite{poelwijk2019learning}. However, these works leave untreated the closure of simplicial complexes under inclusion.
We show that this closure induces a nested structure on the signal spaces of consecutive levels, and develop a signal processing treatment that connects it to signal filtering and reconstruction.

\smallskip \noindent \textbf{Contribution.}
We study the Unoriented Topological Signal Processing (UTSP) framework, which replaces oriented boundary operators by unoriented incidence matrices and analyzes the structure of an unoriented higher-order signal. Our contributions are three-fold.

\smallskip \noindent \textbf{(C1)} We show that unoriented incidence and Laplacian matrices between arbitrary simplicial levels have analogous spectral properties to the graphs they induce on simplices, thus making graph signal processing directly available for unoriented simplicial signals.

\smallskip \noindent \textbf{(C2)} We prove that signals on any simplicial complex admit a nested orthogonal decomposition by \emph{interaction order}, and we show that unoriented Laplacians can only process signals up to a limited interaction order.

\smallskip \noindent \textbf{(C3)} We derive regularizers for signal reconstruction, including one that shrinks each interaction order separately. On real datasets, these regularizers outperform oriented and naive baselines, with larger gains where the energy is more unevenly distributed across orders. 

\begin{figure}
    \centering
    \vspace{-.7cm}
    \includegraphics[width=\linewidth]{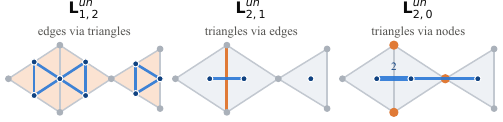}
    \caption{Graphs induced by different unoriented Laplacians $\mtL_{p,q}^{un}$ (in blue). Nodes correspond to simplices at level $p$, edges connect them if they share a simplex at level $q$ (shared simplices are orange).}
    \label{fig:laplacian_ex}
\end{figure}

\section{Background and Unoriented Operators}
\label{sec:operators}

In this section, we first recall simplicial complexes and oriented TSP.
Then, we introduce the unoriented incidence and Laplacian matrices
and discuss their spectral properties.
 
\smallskip \noindent \textbf{Simplicial complexes.}
Let $V$ be a finite vertex set. A $p$-simplex is a subset of $V$ with $p+1$ elements, and a simplicial complex $\mathcal{K}$ is a collection of simplices closed under inclusion: $\sigma\in\mathcal K$ and $\tau\subseteq\sigma$ imply $\tau\in\mathcal K$. A simplex $\tau$ is a \textit{face} of $\sigma$ if $\tau\subseteq\sigma$ and a \textit{co-face} of $\sigma$ if $\sigma\subseteq\tau$.
We write $\mathcal K_p$ for the set of $p$-simplices and $N_p=|\mathcal K_p|$ for their number; a \emph{signal at level $p$} is a vector in $\mathbb{R}^{N_p}$. Oriented representations, following algebraic topology~\cite{carlsson2009topology}, assign each simplex an orientation and define boundary matrices $\mtB_p\in\{0,\pm1\}^{N_{p-1}\times N_p}$ satisfying $\mtB_{p}\mtB_{p+1}=\mathbf 0$. Here we drop orientations.
 
\smallskip \noindent \textbf{Unoriented incidence.}
For levels $p<q$, the unoriented incidence matrix $\mtQ_{p,q}\in\{0,1\}^{N_p\times N_q}$ has $(\mtQ_{p,q})_{\sigma,\tau}=1$ iff $\sigma\subset\tau$. We use the conventions $\mtQ_{-1,q}:=\mathbf 1^\top$ (the empty face lies in every simplex) and $\mtQ_{q,q}:=\mtI$. Unlike oriented boundaries, products of incidence matrices do not vanish. Instead, they satisfy the following identity:
\begin{equation}\label{eq:incidence_id}
    \mtQ_{p,r}\mtQ_{r,q}=\binom{q-p}{r-p}\mtQ_{p,q}, \quad p<r\leq q.
\end{equation}
That is, going from level $p$ to level $q$ through an intermediate level $r$ equals going directly, up to a combinatorial constant. This result is the foundation of the interaction-order decomposition in Theorem~\ref{thm:filtration}.

 
 
\smallskip \noindent \textbf{Unoriented Laplacians.}
For levels $p\neq q$ we define the unoriented Laplacian acting on level $p$ through level $q$ as 
\begin{equation}\label{eq:Lun}
    \mtL^{un}_{p,q} = \begin{cases}
        \mtQ_{p,q}\mtQ_{p,q}^\top & p<q,\\
        \mtQ_{q,p}^\top\mtQ_{q,p} & p>q.
    \end{cases}
\end{equation}
Its diagonal counts the $q$-simplices incident to a $p$-simplex $\sigma$ and its off-diagonal $(\sigma,\tau)$ entry counts the $q$-simplices incident to both $\sigma, \tau$, i.e., it encodes connectivity at level $p$ via shared simplices at level $q$ (see Figure~\ref{fig:laplacian_ex} for examples). Therefore, $\mtL^{un}_{p,q}$ is the weighted adjacency matrix of the graph where simplices at level $p$ are nodes and edges connect two simplices if they share a simplex at level $q$, plus a diagonal term. 
Consequently, the spectral properties of any single $\mtL^{un}_{p,q}$ follow from adjacency-based graph signal processing~\cite{sandryhaila2014discrete,sandryhaila2013discrete}. First, it is symmetric positive semidefinite. Second, being entrywise non-negative, on a connected graph, its top eigenvector is single-signed and concentrates on the densest structures of that graph, i.e., cliques and hubs~\cite{cvetkovic2007signless}. Third, its zero eigenvalues correspond to signals whose values cancel when summed over every $q$-simplex (for $\mtL^{un}_{1,0}$, alternating signs along even cycles or pairs of odd cycles).
Figure~\ref{fig:edge_eig} illustrates these facts on a small complex. 
Operators between non-adjacent levels, such as the triangle-through-node Laplacian $\mtL^{un}_{2,0}$, have no oriented counterpart, since $\mtB_1\mtB_2=\mathbf 0$.
What single-graph theory does not describe is the relation \emph{between} the operators of the family $\{\mtL^{un}_{p,q}\}_{q}$, which we study next.  

\begin{figure}[t]
    \centering
    \vspace{-.7cm}
    \includegraphics[width=\linewidth]{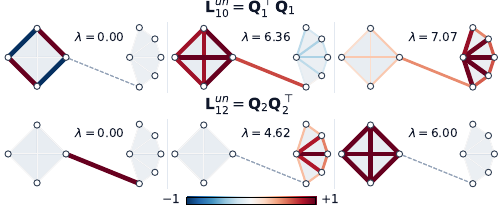}
    \caption{Top and bottom eigenvectors of the unoriented edge Laplacians $\mtL^{un}_{1,0},\mtL^{un}_{1,2}$. For both, the top eigenvectors localize on cliques and hubs. The bottom eigenvector of $\mtL^{un}_{1,0}$ is highly alternating on a 4-cycle, while that of $\mtL^{un}_{1,2}$ centers on a disconnected edge. All 3-cliques are filled by a triangle.}
    \label{fig:edge_eig}
\end{figure}

\section{Interaction-Order Decomposition}
\label{sec:order}

We now consider signals $\vcs_p \in \mathbb{R}^{N_{p}}$, which assign a value to each simplex at level $p$.
The incidence operators allow lifting signals from one level to another. For example, the operation $\mtQ_{p,q}^\top\vcs_p$ assigns to each simplex $\sigma \in \mathcal{K}_q$ the sum of the signals of its faces $\rho$ at level $p$, where $\rho \in \mathcal{K}_p, \rho \subset \sigma$. In this way, one can generate edge signals as aggregations (or \textit{lifts}) of node signals, triangle signals as lifts of node and edge signals, and so on. 
Using this approach, one can decompose a signal $\vcs_p$ into contributions of lower-level simplices, as we formalize next. 

We begin by defining the space $\mathcal V_k:=\range(\mtQ_{k,p}^\top)\subseteq\mathbb{R}^{N_p}$ for $k=-1,0,\dots,p$, which identifies the space of level-$p$ signals that are lifts of signals on $k$-simplices, i.e., sums of contributions of $(k{+}1)$-entities. 
We say that a signal in $\mathcal V_k$ has \textit{interaction order} at most $k$, i.e., it can be written by contributions of smaller simplices up to level $k$. By the conventions in Section~\ref{sec:operators}, we get $\mathcal V_{-1}=\mathrm{span}(\mathbf 1)$ and $\mathcal V_{p}=\mathbb{R}^{N_p}$.
With this in place, we state the decomposition in the following theorem.


\begin{theorem}[Interaction-order decomposition]\label{thm:filtration}
For any simplicial complex and any level $p$:
\begin{enumerate}[label=(\alph*),nosep,leftmargin=1.6em]
\item $\mathcal V_{-1}\subseteq\mathcal V_0\subseteq\cdots\subseteq\mathcal V_{p-1}\subseteq\mathcal V_{p}=\mathbb{R}^{N_p}$;
\item for every $q<p$, $\ \range(\mtL^{un}_{p,q})=\mathcal V_{q}$ and $\ \ker(\mtL^{un}_{p,q})=\mathcal V_{q}^{\perp}$.
\end{enumerate}
\end{theorem}

\begin{proof}
(a) \eqref{eq:incidence_id} with levels $(k,k+1,p)$ and $0\leq k\leq p-1$ gives $\mtQ_{k,k+1}\mtQ_{k+1,p}=(p-k)\,\mtQ_{k,p}$. Transposing, $\mtQ_{k,p}^\top=\tfrac{1}{p-k}\mtQ_{k+1,p}^\top\mtQ_{k,k+1}^\top$, so every vector in $\mathcal V_k$ is the lift of a vector on $(k{+}1)$-simplices, i.e.,\ $\mathcal V_k\subseteq\mathcal V_{k+1}$.
For $k=-1$ use $\mtQ_{0,p}^\top\mathbf 1=(p+1)\mathbf 1$. 
(b) With $\mathbf A=\mtQ_{q,p}$, $\mtL^{un}_{p,q}=\mathbf A^\top\mathbf A$ has $\ker(\mathbf A^\top\mathbf A)=\ker\mathbf A=\range(\mathbf A^\top)^\perp=\mathcal V_q^\perp$, hence $\range(\mtL^{un}_{p,q})=\ker(\mtL^{un}_{p,q})^\perp=\mathcal V_q$.
\end{proof}

Theorem~\ref{thm:filtration} proves that, for a fixed level $p$, the space of signals with interaction order at most $p-2$ is a subspace of that of signals with interaction order at most $p-1$, and so on.
Since the spaces are nested, $\mathbb{R}^{N_p}$ splits orthogonally into the \emph{bands} $\mathcal W_k:=\mathcal V_k\cap\mathcal V_{k-1}^\perp$, $k=-1,\dots,p$ with $\mathcal{V}_{-2}=\{0\}$. Writing $\mtP_k$ for the orthogonal projector onto $\mathcal W_k$, every level-$p$ signal decomposes uniquely as
\begin{equation}\label{eq:decomp}
    \vcx=\sum_{k=-1}^{p}\mtP_k\vcx,\qquad \pi_k(\vcx):=\frac{\|\mtP_k\vcx\|_2^2}{\|\vcx\|_2^2},
\end{equation}
where $\mtP_k\vcx$ is the \emph{pure order-$k$} component and $\pi_k$ the fraction of energy at order $k$. The projectors are computed by least squares on the incidence matrices ($\mtP_{\mathcal V_k}=\mtQ^\top(\mtQ\mtQ^\top)^{+}\mtQ$ with $\mtQ=\mtQ_{k,p}$, then $\mtP_k=\mtP_{\mathcal V_k}-\mtP_{\mathcal V_{k-1}}$).

\smallskip \noindent \textbf{Interpretation.}
Consider a triangle signal. Part of it may be explained by each node having a strength and the triangle value being their sum (order $0$); a further part by each pair contributing (order $1$); what is left is genuinely three-way (order $2$). Theorem~\ref{thm:filtration}(a) states that these explanations are nested, so the split into ``main effects, pairwise effects, genuinely higher-order effects'' is well defined on any complex. The reason is \eqref{eq:incidence_id}: a node-additive triangle signal is also pair-additive, since assigning each edge half the sum of its endpoint strengths reproduces the same result. Nesting requires closure under faces, so this structure is specific to simplicial complexes. Figure~\ref{fig:bands} shows a visualization of this signal decomposition.
Theorem~\ref{thm:filtration}(b) then identifies exactly what each unoriented Laplacian sees: $\mtL^{un}_{p,q}$ acts on the components of order at most $q$ and is identically zero on all higher orders. For example, the triangles-through-nodes Laplacian $\mtL^{un}_{2,0}$ sees the node-additive part only; the triangles-through-edges Laplacian $\mtL^{un}_{2,1}$ sees orders $\leq1$; neither touches the pure three-way band. 
For oriented operators there is no such statement since their ranges are mutually orthogonal (the Hodge decomposition) rather than nested, and they mix the bands.

\smallskip \noindent \textbf{Comparison with oriented TSP.}
Two main differences follow from the definitions. 
(i) The oriented framework organizes signals by flow structure (gradient, curl, harmonic); the unoriented one organizes them by interaction order. 
(ii) The kernels of the Hodge Laplacians identify topological invariants (e.g., number of connected components or topological holes), while the kernel of the unoriented Laplacian corresponds to different structures: $\dim\ker\mtL^{un}_{0,1}$ counts bipartite components; $\ker\mtL^{un}_{1,0}$ is spanned by alternating signals on even cycles and pairs of odd cycles joined by a path~\cite{godsil2001algebraic} (cf. Figure~\ref{fig:edge_eig}).

 \begin{figure}[t]
    \centering
    \includegraphics[width=.8\linewidth]{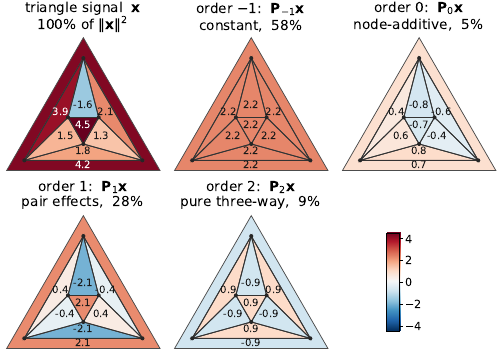}
    \caption{ Interaction-order decomposition of a triangle signal $\vcx$ (Theorem~\ref{thm:filtration}). $\vcx$ is decomposed into four orthogonal components: one constant, one node-additive $\mtP_0\vcx$, one pair-additive $\mtP_1\vcx$, and the remainder three-way effect. The triangle below the complex is shown around it. }
    \label{fig:bands}
\end{figure}

\section{Filters and Signal Reconstruction}
\label{sec:applications}

In this section, we present practical applications of this framework.
Throughout this section the signal $\vcs$ lives at level $p$ and is processed through a lower level $q<p$; by Theorem~\ref{thm:filtration}(b) the Laplacian $\mtL^{un}_{p,q}$ then acts on the components of $\vcs$ of order at most $q$ and is zero on the components of order $q+1,\dots,p$. For a triangle signal ($p=2$) related through nodes ($q=0$), the operator acts on the node-additive part, a subspace of dimension at most $N_0$ inside $\mathbb{R}^{N_2}$, and ignores the pair-explained and genuinely three-way content.
 
\smallskip \noindent \textbf{Filters.}
Let $\mtL^{un}_{p,q}=\mtV\mtLambda\mtV^\top$ with eigenvalues $\lambda_1\geq\lambda_2\cdots\geq \lambda_{N_p}\geq0$. A linear shift-invariant filter at level $p$ is $\vchs=\mtV g(\mtLambda)\mtV^\top\vcs$ for a spectral response $g$. Since $\mtL^{un}_{p,q}$ is a non-negative shift, we define frequency as the distance from the top eigenvalue $\omega=\lambda_1-\lambda$~\cite{sandryhaila2014discrete}, so that the top eigenvector (cohesive, single-signed) has frequency $0$ and highly alternating eigenvectors have high frequency. 
We next show that filters of this form are the solution to signal denoising problems.
 
\smallskip \noindent \textbf{Reconstruction.}
Let $\vcs\in\mathbb{R}^{N_{p}}$ and $\vcy=\mtM(\vcs+\vcn)$, where $\vcn$ is zero-mean white noise and $\mtM$ is a diagonal $0/1$ sampling mask ($\mtM=\mtI$ for denoising). We estimate
\begin{equation}\label{eq:recon}
    \hat{\vcs}=\arg\min_{\vcx}\ \|\mtM(\vcy-\vcx)\|_2^2+\alpha\,\vcx^\top\mtR\,\vcx+\gamma \vcx^\top \mtGamma \vcx,
\end{equation}
with $\mtR\succeq0$ a regularizer encoding the structural prior on $\vcs$, and $\mtGamma \succeq 0$ a secondary regularizer: $\mtGamma=\mtI$ recovers the standard norm-2 penalty. 
The closed form solution is $\hat{\vcs}=(\mtM+\alpha\mtR+\gamma\mtGamma)^{-1}\vcy$.
For $\mtM=\mtGamma=\mtI$ and $\mtR=h(\mtL^{un}_{p,q})$ the solution is the filter $g(\lambda)=1/(1+\gamma+\alpha h(\lambda))$. Different choices for $\mtR,\mtGamma$ follow from theory.
 
\smallskip \noindent \textit{(1) Cohesion regularizer.}
If $\vcs$ is expected to concentrate on cliques or hubs, set $\mtR=\lambda_1\mtI-\mtL^{un}_{p,q}$, i.e.\ $g(\lambda)=1/(1+\gamma+\alpha(\lambda_1-\lambda))$. This is a soft low-pass in $\omega$: the top eigenvector is shrunk by $1/(1+\gamma)$, and every other eigenvector is shrunk further in proportion to its distance $\lambda_1-\lambda$ from the top.
 
 
\smallskip \noindent \textit{(2) Interaction-order regularizer.}
The cohesion regularizer cannot tell the interaction bands apart. The components of $\vcs$ of order above $q$ lie in $\ker\mtL^{un}_{p,q}$, where the penalty reduces to $\lambda_1\mtI$: in denoising they are all shrunk by the same scalar whatever they contain, and in imputation, since $\mtL^{un}_{p,q}$ is the only coupling between simplices, no observed value propagates into them. The bands of order at most $q$ are treated through the spectrum of $\mtL^{un}_{p,q}$, which mixes them. Genuinely high-order content, e.g., the pure three-way part of a triangle signal, is therefore neither denoised nor inferred, and no choice of $q$ allows shrinking one interaction band independently of another.
To act on every band, we use the projectors of~\eqref{eq:decomp} and define the \textit{interaction-order regularizer} as
\begin{equation}\label{eq:Rord}
    \mtR_{\mathrm{ord}}=\sum_{k=-1}^{p}\beta_k\,\mtP_k,\qquad \beta_k\geq0,
\end{equation}
with band weights $\beta_k$, whose Tikhonov solution for $\mtM=\mtGamma=\mtI$ is $\hat{\vcs}=\sum_k(1+\gamma+\alpha\beta_k)^{-1}\mtP_k\vcy$: each interaction band of the level-$p$ signal, including the top one, is shrunk by its own factor. We consider two profiles for $\beta_k$: a smooth profile such as $\beta_k=(k+1)^\nu$, where a larger $\nu$ shrinks higher interaction orders more, or a hard cut $\mathrm{cut}_m$ where $\beta_k=0$ for $k<m$ and $\beta_k=1$ for $k\ge m$, which leaves the bands below order $m$ shrunk only by $\gamma$ and attenuates the rest by $1/(1+\gamma+\alpha)$.



\smallskip \noindent \textit{(3) Band-aware secondary regularizer.}
We propose the alternative secondary regularizer $\mtGamma=\mtG + \epsilon/\gamma \mtI$, where $\epsilon=10^{-10}$ ensures a solution to \eqref{eq:recon} 
and $\mtG:=\mtQ_{0,p}^\top(\mtQ_{0,p}\mtQ_{0,p}^\top)^{+2}\mtQ_{0,p}=\big((\mtQ_{0,p}^\top)^{+}\big)^{\!\top}(\mtQ_{0,p}^\top)^{+}$, so that $\vcx^\top\mtG\vcx=\|\vcc^\star\|^2$ with $\vcc^\star$ the smallest node signal satisfying $\mtQ_{0,p}^\top\vcc^\star=\mtP_{\mathcal V_0}\vcx$. That is, $\mtG$ penalizes the node-level explanation of $\vcx$ and vanishes iff $\vcx\perp\mathcal V_0$. The plain ridge instead gives, for $\vcx=\mtQ_{0,p}^\top\vcc\in\mathcal V_0$, $\|\vcx\|^2=\vcc^\top\mtL^{un}_{0,p}\vcc$, which weights each node strength by the number of $p$-simplices containing it.
This matters in imputation when the observed simplices are fewer than the nodes: the node-additive fit is then underdetermined, and $\|\vcx\|^2$ resolves the ambiguity by favouring nodes that appear in few $p$-simplices, whereas $\mtG$ acts as a ridge penalty on the node coefficients themselves and picks the most parsimonious set of node strengths consistent with the data.

\begin{figure*}[t]
\vspace{-.7cm}
\centering
\includegraphics[width=\textwidth]{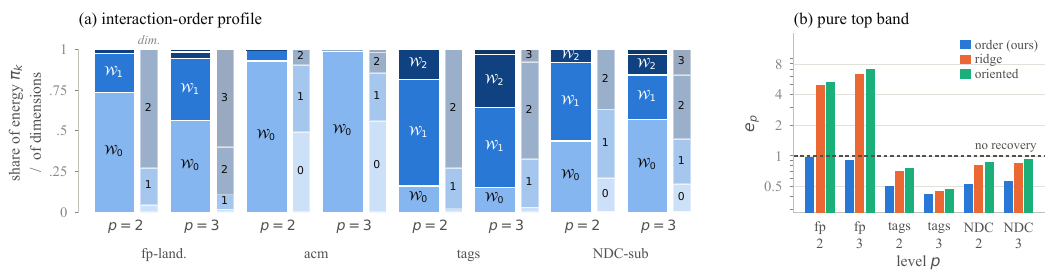}
\footnotesize\caption{(a) Energy per band ($\pi_k$, left) against its share of the
dimensions (right). (b) Error in the top band after denoising ($\lVert\mtP_p(\hat{\vcs}-\vcs)\rVert_2^2/\lVert\mtP_p\vcs\rVert_2^2$); values above $1.0$ mean the estimator injects error into it. We omit \texttt{acm} since the top band holds $\leq0.05\%$ of the energy.}
\label{fig:exp}
\end{figure*}

\begin{table*}[t]
\centering
\setlength{\tabcolsep}{4.5pt}
\renewcommand{\arraystretch}{0.95}
\scriptsize
\caption{NRMSE at $\sigma=0.5$ / $50\%$ missing, mean $\pm$ s.d.\ over $20$ trials;
best per block in \textbf{bold}. $\dagger$ marks
the two proposed regularizers; the rest are baselines. cohesion and oriented are at
their best variant per cell (over $q$, and over the Hodge Laplacians).}
\label{tab:recon}
\begin{tabular}{l c c c c c @{\hskip 12pt} c c c c}
\toprule
 & & \multicolumn{4}{c}{denoising, $\sigma=0.5$}
 & \multicolumn{4}{c}{imputation, $50\%$ missing} \\
\cmidrule(lr){3-6}\cmidrule(lr){7-10}
dataset & $p$ & ridge & oriented & cohesion$^\dagger$ & order$^\dagger$
        & n.-mean & oriented & cohesion$^\dagger$ & order$^\dagger$ \\
\midrule
\texttt{fp-landscape} & $2$ & $0.446_{\pm.017}$ & $0.446_{\pm.018}$ & $0.427_{\pm.014}$ & $\mathbf{0.298_{\pm.015}}$
                      & $0.708_{\pm.016}$ & $0.880_{\pm.044}$ & $0.728_{\pm.016}$ & $\mathbf{0.279_{\pm.045}}$ \\
\texttt{fp-landscape} & $3$ & $0.445_{\pm.009}$ & $0.445_{\pm.009}$ & $0.411_{\pm.009}$ & $\mathbf{0.273_{\pm.006}}$
                      & $0.678_{\pm.016}$ & $0.963_{\pm.028}$ & $0.699_{\pm.017}$ & $\mathbf{0.277_{\pm.016}}$ \\
\texttt{acm-coauth} & $2$ & $0.446_{\pm.010}$ & $0.437_{\pm.011}$ & $0.441_{\pm.010}$ & $\mathbf{0.390_{\pm.011}}$
                      & $0.642_{\pm.035}$ & $0.898_{\pm.050}$ & $0.822_{\pm.030}$ & $\mathbf{0.620_{\pm.033}}$ \\
\texttt{acm-coauth} & $3$ & $0.444_{\pm.016}$ & $0.446_{\pm.016}$ & $0.440_{\pm.016}$ & $\mathbf{0.365_{\pm.015}}$
                      & $0.675_{\pm.144}$ & $0.878_{\pm.135}$ & $0.785_{\pm.144}$ & $\mathbf{0.632_{\pm.152}}$ \\
\texttt{tags-math} & $2$ & $0.447_{\pm.009}$ & $0.445_{\pm.009}$ & $0.443_{\pm.009}$ & $\mathbf{0.395_{\pm.008}}$
                      & $0.899_{\pm.022}$ & $0.954_{\pm.033}$ & $0.877_{\pm.023}$ & $\mathbf{0.645_{\pm.021}}$ \\
\texttt{tags-math} & $3$ & $0.446_{\pm.009}$ & $0.445_{\pm.008}$ & $0.442_{\pm.009}$ & $\mathbf{0.431_{\pm.009}}$
                      & $0.950_{\pm.034}$ & $0.985_{\pm.045}$ & $0.898_{\pm.032}$ & $\mathbf{0.839_{\pm.040}}$ \\
\texttt{NDC-subst.} & $2$ & $0.449_{\pm.021}$ & $0.448_{\pm.022}$ & $0.446_{\pm.021}$ & $\mathbf{0.423_{\pm.021}}$
                      & $0.997_{\pm.072}$ & $0.981_{\pm.069}$ & $0.953_{\pm.054}$ & $\mathbf{0.906_{\pm.066}}$ \\
\texttt{NDC-subst.} & $3$ & $0.446_{\pm.019}$ & $0.446_{\pm.021}$ & $0.443_{\pm.019}$ & $\mathbf{0.420_{\pm.021}}$
                      & $1.012_{\pm.073}$ & $1.010_{\pm.067}$ & $0.958_{\pm.051}$ & $\mathbf{0.864_{\pm.074}}$ \\
\bottomrule
\end{tabular}
\end{table*}

\section{Experiments}
\label{sec:experiments}

\smallskip\noindent\textbf{Data and preprocessing.}
We use four complexes with signals at levels $2$ and $3$ (i.e., triangles and tetrahedra). In \texttt{fp-landscape} nodes are the $13$ mutated sites of a fluorescent protein, a $k$-simplex is the genotype carrying that set of mutations, and the signal is its measured brightness~\cite{poelwijk2019learning}; every subset was measured, so the complex is the full simplex ($N_{0\ldots3}=13,78,286,715$). Two are hypergraph benchmarks~\cite{benson2018simplicial}: in \texttt{tags-math}, nodes are tags and simplices are the sets of tags applied to questions on math.stackexchange.com, while in \texttt{NDC-substances} each simplex is a drug and the nodes are the substances making it up. For both, the complex is the downward closure of the observed hyperedges and the level-$k$ signal is the count of events whose participant set is \emph{exactly} that $k$-simplex, transformed as $\log(1+\text{count})$. 
In
\texttt{acm-coauth} the nodes are authors and a $k$-simplex is a set of $k+1$
co-authors of a paper~\cite{wang2019heterogeneous}; the signal is the $k$-way co-moment
of the TF-IDF~\cite{sparck1972statistical} of their term vectors, i.e., the vocabulary the whole group shares. 
We select the maximum hyperedge size to ensure that the $p=3$ top band remains non-empty. This results in the sizes $6,6,8$ for \texttt{tags-math}, \texttt{NDC-substances} and \texttt{acm-coauth} respectively. Then, we subsample nodes to maximise $N_2/N_0$ subject to $N_2\in[200,1200]$, which results in the values $N_2=1093,204,689$. Signals are standardised. Our code is available online\footnote{\url{https://github.com/andrea-cavallo-98/Unoriented_TSP}}.

\smallskip\noindent\textbf{Setup.}
We solve~\eqref{eq:recon} for denoising ($\mtM=\mtI$) and imputation (error on the
removed entries), scoring $\mathrm{RMSE}/\mathrm{std}(\vcs)$ over $20$ trials at noise level
$\sigma=0.5$ and $50\%$ missing. Every $\mtR$ has unit spectral
radius; $\alpha$ spans $[10^{-3},10^{3}]$ and $\gamma\in\{10^{-3},10^{-1},10\}$, chosen
by SURE~\cite{stein1981estimation} for denoising and a $25\%$ validation split for imputation. We compare two naive baselines (\textbf{ridge} for denoising, \textbf{neighborhood mean} for imputation); the \textbf{oriented} Tikhonov regularizer (every Hodge Laplacian, low- and
high-pass); \textbf{cohesion} for every $q<p$; and the \textbf{order} regularizer in~\eqref{eq:Rord} with $\beta\in\{(k+1)^4,\mathrm{cut}_1,\mathrm{cut}_2\}$. 
In imputation we also select $\mtGamma$ between $\mtI, \mtG$, while for denoising we fix $\mtGamma=\mtI$.
For each method, we report the results of the best Laplacian. 

\smallskip\noindent\textbf{Results.}
The order regularizer is best in all cells of Table~\ref{tab:recon}. The gains w.r.t. the oriented estimator 
show that the correct assumption on signal orientedness plays a relevant role in reconstruction quality.
The energy per band differs sharply across the four complexes (Fig.~\ref{fig:exp}a), and the advantage of the order regularizer over the naive baseline tracks how far the energy profile departs from the dimension profile. In particular, 
it is largest on \texttt{fp-landscape}, where 74\% of the energy occupies 4\% of the dimensions, and smallest on \texttt{tags-math} $p{=}3$ and \texttt{NDC-substances}, where the two nearly coincide.
Considering the top band energy (Fig.~\ref{fig:exp}b), the order regularizer is the only one to recover it where it has energy 
and to avoid injecting error where it is noise-dominated. 
Finally, the secondary regularizer $\mtG$ is selected in $66\%$ of \texttt{acm-coauth} imputations, where $\dim\mathcal V_0$ exceeds the observed count, against $7\%$ on \texttt{fp-landscape} ($\dim\mathcal V_0=13$ against $143$ observations). That is, $\mtG$ helps when the node-additive subspace is not identified by the observations.

\section{Conclusion}
We studied the Unoriented Topological Signal Processing (UTSP) framework, which replaces oriented boundary operators with unoriented incidence matrices and Laplacians. We characterized their spectral properties and established an orthogonal nested decomposition by interaction order, from which we built order-specific regularizers for signal denoising and imputation. Empirical evaluations on real-world datasets confirm that the unoriented framework largely outperforms oriented competitors and interaction order-specific regularization helps for signals distributed unevenly across orders. 

\section{Acknowledgments}

This work was supported in part by the TU Delft AI Labs programme, NWO OTP GraSPA proposal \#19497, NWO VENI proposal 222.032, and by the SURE-AI Centre grant \#357482, Research Council of Norway. 
Claude Opus 5 (Anthropic) was used to assist in writing the manuscript and coding the simulations. The authors take full responsibility for the results of this paper.

\section{Compliance with Ethical Standards}

This is a numerical simulation study for which no ethical approval was required.

\bibliographystyle{IEEEbib}
\bibliography{citations}

@article{chen2026bounding,
  title={Bounding the largest eigenvalue of signless Laplace operator on simplicial complexes},
  author={Chen, Xiaodan and Kou, Yongfang},
  journal={Graphs and Combinatorics},
  volume={42},
  number={4},
  pages={57},
  year={2026},
  publisher={Springer}
}

@article{fan2026signless,
  title={Signless Laplacian Spectral Radius and Link Homology of Simplicial Complexes},
  author={Fan, Yi-Zheng and Zhang, Huan-Zhi},
  journal={arXiv preprint arXiv:2606.22825},
  year={2026}
}

@article{sardellitti2026learning,
  title={Learning Laplacian Forms for Graph Signal Processing via the Deformed Laplacian},
  author={Sardellitti, Stefania},
  journal={arXiv preprint arXiv:2604.00728},
  year={2026}
}

@article{cvetkovic2007signless,
  title={Signless Laplacians of finite graphs},
  author={Cvetkovi{\'c}, Drago{\v{s}} and Rowlinson, Peter and Simi{\'c}, Slobodan K},
  journal={Linear Algebra and its applications},
  volume={423},
  number={1},
  pages={155--171},
  year={2007},
  publisher={Elsevier}
}

@inproceedings{reddy2023clustering,
  title={Clustering with simplicial complexes},
  author={Reddy, Thummaluru Siddartha and Chepuri, Sundeep Prabhakar and Borgnat, Pierre},
  booktitle={2023 31st European Signal Processing Conference (EUSIPCO)},
  pages={1609--1613},
  year={2023},
  organization={IEEE}
}

@inproceedings{fuchsgruber2025graph,
  title={Graph neural networks for edge signals: Orientation equivariance and invariance},
  author={Fuchsgruber, Dominik and Postuvan, Tim and G{\"u}nnemann, Stephan and Geisler, Simon Markus},
  booktitle={International Conference on Learning Representations},
  volume={2025},
  pages={40915--40948},
  year={2025}
}

@article{battiston2020networks,
  title={Networks beyond pairwise interactions: Structure and dynamics},
  author={Battiston, Federico and others},
  journal={Physics reports},
  volume={874},
  pages={1--92},
  year={2020},
  publisher={Elsevier}
}

@article{benson2018simplicial,
  title={Simplicial closure and higher-order link prediction},
  author={Benson, Austin R and others},
  journal={Proceedings of the National Academy of Sciences},
  volume={115},
  number={48},
  pages={E11221--E11230},
  year={2018},
  publisher={National Academy of Sciences}
}

@article{barbarossa2020topological,
  title={Topological signal processing over simplicial complexes},
  author={Barbarossa, Sergio and Sardellitti, Stefania},
  journal={IEEE Transactions on Signal Processing},
  volume={68},
  pages={2992--3007},
  year={2020},
  publisher={IEEE}
}

@article{schaub2020random,
  title={Random walks on simplicial complexes and the normalized Hodge 1-Laplacian},
  author={Schaub, Michael T and others},
  journal={SIAM Review},
  volume={62},
  number={2},
  pages={353--391},
  year={2020},
  publisher={SIAM}
}

@article{isufi2025topological,
  title={Topological signal processing and learning: Recent advances and future challenges},
  author={Isufi, Elvin and others},
  journal={Signal Processing},
  volume={233},
  pages={109930},
  year={2025},
  publisher={Elsevier}
}

@article{munkres2000topology,
  title={Topology. featured titles for topology series},
  author={Munkres, James R},
  journal={Prentice Hall, Incorporated},
  volume={812},
  number={813},
  pages={24},
  year={2000}
}

@article{carlsson2009topology,
  title={Topology and data},
  author={Carlsson, Gunnar},
  journal={Bulletin of the American mathematical society},
  volume={46},
  number={2},
  pages={255--308},
  year={2009}
}

@inproceedings{bodnar2021weisfeiler,
  title={Weisfeiler and lehman go topological: Message passing simplicial networks},
  author={Bodnar, Cristian and others},
  booktitle={International conference on machine learning},
  pages={1026--1037},
  year={2021},
  organization={PMLR}
}

@article{hajij2022topological,
  title={Topological deep learning: Going beyond graph data},
  author={Hajij, Mustafa and others},
  journal={arXiv preprint arXiv:2206.00606},
  year={2022}
}

@article{zhang2019introducing,
  title={Introducing hypergraph signal processing: Theoretical foundation and practical applications},
  author={Zhang, Songyang and Ding, Zhi and Cui, Shuguang},
  journal={IEEE Internet of Things Journal},
  volume={7},
  number={1},
  pages={639--660},
  year={2019},
  publisher={IEEE}
}

@article{chan2018spectral,
  title={Spectral properties of hypergraph laplacian and approximation algorithms},
  author={Chan, T-H Hubert and Louis, Anand and Tang, Zhihao Gavin and Zhang, Chenzi},
  journal={Journal of the ACM (JACM)},
  volume={65},
  number={3},
  pages={1--48},
  year={2018},
  publisher={ACM New York, NY, USA}
}

@inproceedings{schaub2018flow,
  title={Flow smoothing and denoising: Graph signal processing in the edge-space},
  author={Schaub, Michael T and Segarra, Santiago},
  booktitle={2018 IEEE Global Conference on Signal and Information Processing (GlobalSIP)},
  pages={735--739},
  year={2018},
  organization={IEEE}
}

@inproceedings{wang2019heterogeneous,
  title={Heterogeneous graph attention network},
  author={Wang, Xiao and others},
  booktitle={The world wide web conference},
  pages={2022--2032},
  year={2019}
}

@article{sandryhaila2013discrete,
  title={Discrete signal processing on graphs},
  author={Sandryhaila, Aliaksei and Moura, Jos{\'e} MF},
  journal={IEEE transactions on signal processing},
  volume={61},
  number={7},
  pages={1644--1656},
  year={2013},
  publisher={IEEE}
}

@article{schaub2021signal,
  title={Signal processing on higher-order networks: Livin’on the edge... and beyond},
  author={Schaub, Michael T and others},
  journal={Signal Processing},
  volume={187},
  pages={108149},
  year={2021},
  publisher={Elsevier}
}

@article{yang2022simplicial,
  title={Simplicial convolutional filters},
  author={Yang, Maosheng and Isufi, Elvin and Schaub, Michael T and Leus, Geert},
  journal={IEEE Transactions on Signal Processing},
  volume={70},
  pages={4633--4648},
  year={2022},
  publisher={IEEE}
}

@article{cvetkovic2009towards,
  title={Towards a spectral theory of graphs based on the signless Laplacian, I},
  author={Cvetkovi{\'c}, Drago{\v{s}} and Simi{\'c}, Slobodan K},
  journal={Publications de l'Institut Mathematique},
  volume={85},
  number={105},
  pages={19--33},
  year={2009},
  publisher={Matemati{\v{c}}ki institut SANU}
}

@article{zhou2006learning,
  title={Learning with hypergraphs: Clustering, classification, and embedding},
  author={Zhou, Dengyong and Huang, Jiayuan and Sch{\"o}lkopf, Bernhard},
  journal={Advances in neural information processing systems},
  volume={19},
  year={2006}
}

@article{efron1981jackknife,
  title={The jackknife estimate of variance},
  author={Efron, Bradley and Stein, Charles},
  journal={The Annals of Statistics},
  pages={586--596},
  year={1981},
  publisher={JSTOR}
}

@article{poelwijk2019learning,
  title={Learning the pattern of epistasis linking genotype and phenotype in a protein},
  author={Poelwijk, Frank J and Socolich, Michael and Ranganathan, Rama},
  journal={Nature communications},
  volume={10},
  number={1},
  pages={4213},
  year={2019},
  publisher={Nature Publishing Group UK London}
}

@article{stein1981estimation,
  title={Estimation of the mean of a multivariate normal distribution},
  author={Stein, Charles M},
  journal={The annals of Statistics},
  pages={1135--1151},
  year={1981},
  publisher={JSTOR}
}

@article{sparck1972statistical,
  title={A statistical interpretation of term specificity and its application in retrieval},
  author={Sparck Jones, Karen},
  journal={Journal of documentation},
  volume={28},
  number={1},
  pages={11--21},
  year={1972},
  publisher={MCB UP Ltd}
}

@book{godsil2001algebraic,
  title={Algebraic graph theory},
  author={Godsil, Christopher David and Royle, Gordon and Godsil, CD},
  volume={207},
  year={2001},
  publisher={Springer New York}
}

@article{liu2026matched,
  title={Matched topological subspace detector},
  author={Liu, Chengen and Tenorio, Victor M and Marques, Antonio G and Isufi, Elvin},
  journal={IEEE Transactions on Signal Processing},
  year={2026},
  publisher={IEEE}
}

@article{sandryhaila2014discrete,
  title={Discrete signal processing on graphs: Frequency analysis},
  author={Sandryhaila, Aliaksei and Moura, Jose MF},
  journal={IEEE Transactions on signal processing},
  volume={62},
  number={12},
  pages={3042--3054},
  year={2014},
  publisher={IEEE}
}

\end{document}